\documentclass[11pt]{article}
\usepackage{hyphenat}

\usepackage{authblk}

\usepackage{graphicx}
\usepackage{caption}
\usepackage{subcaption}
\usepackage{float}

\usepackage{amsmath, amssymb, amsfonts, mathtools, amsthm}
\usepackage{esint}        
\usepackage{mathrsfs}     

\usepackage{algorithm}
\usepackage{algorithmic}

\usepackage{tikz}
\usetikzlibrary{shapes.geometric, arrows, positioning, arrows.meta}

\usepackage{color}
\usepackage{xcolor}
\usepackage{microtype}    
\usepackage{parskip}      
\usepackage{indentfirst}  
\usepackage{geometry}
\usepackage{listings}
\usepackage{tcolorbox}

\usepackage{fancybox}
\usepackage{thmtools}
\usepackage{mdframed}
\usepackage{shadethm}

\usepackage{cite}

\usepackage{fancyhdr}
\usepackage[shortlabels]{enumitem}

\usepackage{lipsum}       
\usepackage{setspace}

\usepackage[pdftex, colorlinks=true, linkcolor=blue, citecolor=blue, urlcolor=blue]{hyperref}

\tikzstyle{startstop} = [
    rectangle, 
    rounded corners, 
    minimum width=3cm, 
    minimum height=1cm,
    text centered, 
    draw=black, 
    fill=red!30
]
\tikzstyle{io} = [
    trapezium, 
    trapezium left angle=70, 
    trapezium right angle=110, 
    minimum width=3cm, 
    minimum height=1cm, 
    text centered, 
    draw=black, 
    fill=blue!30
]
\tikzstyle{process} = [
    rectangle, 
    minimum width=3cm, 
    minimum height=1cm, 
    text centered, 
    draw=black, 
    fill=orange!30
]
\tikzstyle{decision} = [
    diamond, 
    minimum width=3cm, 
    minimum height=1cm, 
    text centered, 
    draw=black, 
    fill=green!30
]
\tikzstyle{arrow} = [
    thick,
    ->,
    >=stealth
]

\newtheorem{thm}{Theorem}[section]
\newtheorem{lemma}[thm]{Lemma}

\newtheorem{proposition}[thm]{Proposition}

\newtheorem{remark}[thm]{Remark}
\newtheorem{assumption}[thm]{Assumption}

\title{\textbf{Resolving gradient pathology in physics-informed epidemiological models}}

\author[1]{Nickson Golooba}
\author[1*]{Woldegebriel Assefa Woldegerima}

\footnotesize{
\affil[1]{Disease-informed Modeling, Methods \& Systems (DIMMS) Lab, Department of Mathematics and Statistics, York University, Toronto, Canada}
}

\date{*\textit{corresponding author}: wassefaw@yorku.ca} 

\begin{document}
\maketitle

\begin{abstract}
Physics-informed neural networks (PINNs) are increasingly used in mathematical epidemiology to bridge the gap between noisy clinical data and compartmental models, such as the susceptible-exposed-infected-removed (SEIR) model. However, training these hybrid networks is often unstable due to competing optimization objectives. As established in recent literature on ``gradient pathology," the gradient vectors derived from the data loss and the physical residual often point in conflicting directions, leading to slow convergence or optimization deadlock. While existing methods attempt to resolve this by balancing gradient magnitudes or projecting conflicting vectors, we propose a novel method, conflict-gated gradient scaling (CGGS), to address gradient conflicts in physics-informed neural networks for epidemiological modelling, ensuring stable and efficient training and a computationally efficient alternative. This method utilizes the cosine similarity between the data and physics gradients to dynamically modulate the penalty weight. Unlike standard annealing schemes that only normalize scales, CGGS acts as a geometric gate: it suppresses the physical constraint when directional conflict is high, allowing the optimizer to prioritize data fidelity, and restores the constraint when gradients align. We prove that this gating mechanism preserves the standard $O(1/T)$ convergence rate for smooth non-convex objectives, a guarantee that fails under fixed-weight or magnitude-balanced training when gradients conflict. We demonstrate that this mechanism autonomously induces a curriculum learning effect, improving parameter estimation in stiff epidemiological systems compared to magnitude-based baselines. Our empirical results show improved peak recovery and convergence over magnitude-based methods. 
\end{abstract}
\section{Introduction}

Mathematical epidemiology has traditionally relied on compartmental models, such as the susceptible-exposed-infected-removed (SEIR) framework, to understand how infectious diseases spread. These models are valuable because they follow strict conservation laws and biological logic \cite{kermack1927contribution}. However, they are often too rigid to capture the complex, noisy reality of actual clinical data. On the other hand, deep learning approaches are excellent at fitting data but often fail to respect biological rules, leading to predictions that may violate basic constraints, such as keeping population counts positive.

In recent years, physics-informed neural networks (PINNs) \cite{raissi2019physics, karniadakis2021physics} have emerged as a way to combine the strict biological rules of the SEIR model with the flexible data-fitting power of neural networks. The standard objective of these frameworks is to force the model to fit the clinical data while simultaneously respecting the differential equations of epidemiology.

However, achieving this balance in practice is mathematically challenging. The network must minimize two different errors at the same time, that is to say, the error from the noisy data and the error from the physical equations. Often, the gradient vector from the data loss points in a different direction than the gradient vector from the physics loss. This phenomenon was formally described by Wang et al. \cite{wang2021understanding} as \textit{Gradient Pathology}. When these gradients conflict, the optimization process can stall or oscillate, failing to find a solution that satisfies both the biology and the data.

There are several existing ways to handle this problem, but they have limitations in this specific context. One common approach is \textit{Learning Rate Annealing} (LRA) \cite{wang2021understanding}, which balances the losses based on the size (magnitude) of their gradients. While this stops one loss from dominating the other, it ignores the direction; it does not tell the optimizer what to do when the gradients are pointing in opposite directions. Another approach is \textit{Gradient Surgery} (PCGrad) \cite{yu2020gradient}, which projects conflicting gradients onto a normal plane. While effective, this method is computationally expensive. Even recent variants like GCond \cite{alves2025gcond}, which attempt to reduce this cost via gradient accumulation, are primarily designed for large-scale computer vision tasks rather than the specific stiffness constraints found in epidemiological differential equations. More recently, researchers have proposed second-order optimization methods, such as SOAP \cite{wang2025gradient}, or conflict-free updates like ConFIG \cite{liu2024config} to align these gradients using curvature information. Some have even attempted to use cosine similarity against the total gradient \cite{wang2024cosine}. While highly accurate, these methods often entail significant computational overhead or magnitude bias compared to standard first-order optimizers like Adam.

To bridge this gap, we propose a novel method, Conflict-Gated Gradient Scaling (CGGS), to address gradient conflicts in physics-informed neural networks for epidemiological modelling, ensuring stable and efficient training. Instead of projecting gradients or calculating complex curvature matrices, we introduce a simple geometric ``gate" into the loss function. Our method calculates the cosine similarity between the data gradient and the physics gradient. If they point in opposite directions (conflict), the method automatically lowers the weight of the physical constraint, allowing the model to learn the data trend first. As the gradients align, the weight increases to enforce the biology. This creates an autonomous curriculum that stabilizes training without the high computational cost of second-order methods. Our method is a dynamic, geometry-based approach to modulate the physics gradient based on gradient alignment.

Our main contributions are as follows:
\begin{enumerate}
    \item We formulate a hybrid loss landscape for SEIR modelling that integrates continuous differential equations with discrete logical constraints to prevent biological violations.
    \item We analyze the specific ``Gradient Pathology" in stiff epidemiological models, showing that directional conflict often leads to optimization deadlock.
    \item We propose the CGGS algorithm. Unlike standard magnitude balancing, CGGS uses geometric alignment to dynamically switch between data-priority and physics-priority modes, ensuring robust convergence even with noisy data.
    \item We prove that CGGS converges to first-order stationary points of the data loss at the standard $O(1/T)$ rate for smooth non-convex objectives (Theorem~\ref{thm:main}), a guarantee that fails under fixed-weight training when gradients conflict (Proposition~\ref{prop:deadlock}). We further confirmed that the proposed CGGS avoids Pareto deadlock and ensures robust convergence even under gradient conflict. 
\end{enumerate}

\section{Mathematical formulation}

We consider the problem of approximating the dynamics of an infectious disease over a time domain $\Omega = [0, T]$. Let the state of the system be represented by a vector-valued function $\mathbf{u}(t) = [S(t), E(t), I(t), R(t)]^T$, which we approximate using a deep neural network parametrized by weights and biases $\theta$, denoted as $\mathbf{u}_\theta(t)$.

\subsection{The compartmental constraint (ODE)}
We adopt the Susceptible-Exposed-Infected-Removed (SEIR) model as our continuous governing equation. This is a standard framework in epidemiology that tracks the flow of the population between four compartments. In the context of scientific machine learning, we do not solve these equations forward in time; rather, we compute the residual error of the network's output against the differential operators.

The residual vector $\mathcal{F}(\mathbf{u}_\theta, t)$ is defined as:
\begin{equation}
    \mathcal{F}(\mathbf{u}_\theta, t) = 
    \begin{pmatrix}
    \frac{dS}{dt} + \beta \frac{SI}{N} \\[4pt]
    \frac{dE}{dt} - \beta \frac{SI}{N} + \sigma E \\[4pt]
    \frac{dI}{dt} - \sigma E + \gamma I \\[4pt]
    \frac{dR}{dt} - \gamma I
    \end{pmatrix},
\end{equation}
where $N$ is the total population, $\beta$ is the transmission rate, $\sigma$ is the incubation rate, and $\gamma$ is the recovery rate. If the network perfectly models the disease dynamics, the squared norm $\|\mathcal{F}\|^2$ should approach zero.

\subsection{The logical constraint (discrete knowledge)}
Standard PINNs rely solely on the ODE residuals to guide the network. However, neural networks are universal approximators that do not naturally respect state-space boundaries. A frequent issue in training is the emergence of ``biological hallucinations,'' such as negative population counts, which are mathematically possible in the unconstrained function space but biologically impossible.

To address this, we introduce a set of logical constraints. We require that all compartments remain non-negative and that the cumulative count of Removed individuals (recoveries plus deaths) is monotonically non-decreasing. We formulate these rules using the Rectified Linear Unit (ReLU) as a penalty function:
\begin{equation}
    \mathcal{L}_{logic}(\theta) = \frac{1}{|\mathcal{T}_{col}|} \sum_{t \in \mathcal{T}_{col}} \left( \sum_{k \in \{S,E,I,R\}} \text{ReLU}(-u_k(t)) + \text{ReLU}\left( R(t) - R(t+\delta t) \right) \right)
\end{equation}
Here, the first term imposes a penalty if any population compartment drops below zero, and the second term imposes a penalty if the Removed population decreases over a time step $\delta t$.

\subsection{The unified optimization problem}
We formulate the training of the network as a multi-objective optimization problem. We seek the parameters $\theta^*$ that minimize the total loss $\mathcal{L}_{total}$:

\begin{equation}\label{eq:total_loss}
    \theta^* = \arg \min_{\theta} \left( \mathcal{L}_{data}(\theta, \mathcal{D}) + \lambda_{phy} \mathcal{L}_{ODE}(\theta) + \lambda_{logic} \mathcal{L}_{logic}(\theta) \right)
\end{equation}

where $\mathcal{L}_{data}$ is the Mean Squared Error (MSE) computed on the sparse clinical observations $\mathcal{D}$, and $\mathcal{L}_{ODE}$ is the mean squared residual of the differential equations computed on a dense set of collocation points. 

Crucially, standard approaches treat the weighting factors $\lambda_{phy}$ and $\lambda_{logic}$ as fixed hyperparameters. However, as we demonstrate in the following section, this static formulation is susceptible to severe optimization failures due to the conflicting spectral properties of the data and physics manifolds.

\section{Gradient pathology and spectral analysis}
\label{sec:gradient_pathology}

The convergence of the optimization problem defined in Eq.~\eqref{eq:total_loss} is strictly dictated by the geometry of the loss landscape. In this section, we analyze the behaviour of the gradient descent update vector to understand why standard training often fails in epidemiological modelling.

Let $\mathbf{g}_{data} = \nabla_\theta \mathcal{L}_{data}$ and $\mathbf{g}_{phy} = \nabla_\theta \mathcal{L}_{ODE}$ represent the gradient vectors associated with the observational data and the physical constraints, respectively.

\subsection{The gradient conflict regimes}
In standard backpropagation, the update direction for the network parameters $\theta$ is determined by the weighted sum of the individual gradients: $\mathbf{g}_{total} = \mathbf{g}_{data} + \lambda \mathbf{g}_{phy}$. However, effective learning requires that these components work constructively. To quantify their interaction, we utilize the cosine similarity metric, $S_{cos}$, a standard measure in multi-task learning \cite{yu2020gradient}:

\begin{equation}\label{eq:cosine}
    S_{cos}(\theta) = \frac{\mathbf{g}_{data} \cdot \mathbf{g}_{phy}}{\| \mathbf{g}_{data} \| \| \mathbf{g}_{phy} \|}.
\end{equation}

A publication in 2025 \cite{wang2025gradient} generalizes the pairwise cosine similarity to multiple gradients (more than two loss terms). One popular alignment score for $n$ gradient vectors $\mathbf{g}_1, \mathbf{g}_2, \cdots, \mathbf{g}_n,$ is:
$$A(\mathbf{g}_1, \mathbf{g}_2, \cdots, \mathbf{g}_n)=  2 \left( \frac{ \sum_{i=1}^n \frac{\mathbf{g}_i}{\|\mathbf{g}_i\|} }{n} \right)^2-1.  $$
Particularly, for $n = 2$, this exactly recovers the standard cosine similarity: $A(\mathbf{g}_1, \mathbf{g}_2) =S_{cos}(\theta)$, given in \eqref{eq:cosine}.

Building on the classifications established in prior literature \cite{yu2020gradient, wang2021understanding}, we distinguish three distinct regimes of optimization behavior based on this spectral alignment:
\begin{itemize}
    \item \textbf{Cooperative regime ($S_{cos} > 0$):} Both gradients share a general descent direction. A step decreases both the data error and the physical residual simultaneously. When there is a high positive cosine  (close to $1$), the two loss terms ``agree" on how to update the parameters, yielding fast and stable convergence.
    \item \textbf{Orthogonal regime ($S_{cos} \approx 0$):} The gradients are uncorrelated. There is no directional agreement. Minimizing one objective does not necessarily aid or hinder the other.
    \item \textbf{Conflicting regime ($S_{cos} < 0$):} This is the pathological case. The data objective and the physical objective demand parameter updates in opposing directions. When $S_{cos}$ is negative,  the gradients pull in conflicting directions, which is a major source of gradient pathology in PINNs, causing slow convergence, oscillations, or getting stuck in poor local minima. Particularly, a cosine close to  $-1$ leads to a complete conflict.
\end{itemize}

\subsection{Pareto stationarity and deadlock}\label{sec:Pareto}

The mathematical problem of the conflicting regime, $S_{cos} < 0$ is not merely empirical, it is a structural consequence of multi-objective optimization. A classical characterization due to D\'{e}sid\'{e}ri \cite{desideri2012multiple} establishes that a point $\theta$ is \textit{Pareto-stationary} with respect to multiple objectives if and only if the origin lies in the convex hull of their gradients (Definition~1.1 and Lemma~1.2 in \cite{desideri2012multiple}). For two loss terms, this reduces to finding $\alpha \in [0,1]$ such that $\alpha\, \mathbf{g}_{data} + (1-\alpha)\, \mathbf{g}_{phy} = \mathbf{0}$.

In PINNs for epidemiology, the fast incubation dynamics ($E \to I$, governed by $\sigma$) and the slow transmission dynamics ($S \to E$, governed by $\beta S I / N$) operate on widely separated timescales. During early training, the network has not yet resolved this scale separation, and the physics residual drives $\mathbf{g}_{phy}$ into directions that oppose $\mathbf{g}_{data}$. When a magnitude-balanced weight $\lambda = \|\mathbf{g}_{data}\|/\|\mathbf{g}_{phy}\|$ happens to match the conflict ratio, the gradients cancel exactly, as we stated in the following proposition.

\begin{proposition}[Pareto deadlock under fixed weights]
\label{prop:deadlock}
Let $\mathbf{g}_{data}$ and $\mathbf{g}_{phy}$ be non-zero gradient vectors satisfying $\mathbf{g}_{data} = -c\, \mathbf{g}_{phy}$ for some $c > 0$. Then the magnitude-balanced weight $$\lambda_{std} = \frac{\|\mathbf{g}_{data}\|}{\|\mathbf{g}_{phy}\|}  $$ yields a zero update:
\[
\mathbf{g}_{total} = \mathbf{g}_{data} + \lambda_{std}\, \mathbf{g}_{phy} = \mathbf{0}.
\]
The optimizer perceives this as a stationary point and halts, despite both $\mathcal{L}_{data}$ and $\mathcal{L}_{ODE}$ remaining large.
\end{proposition}

\begin{proof}
Since $\mathbf{g}_{data} = -c\, \mathbf{g}_{phy}$, we have $\|\mathbf{g}_{data}\| = c\,\|\mathbf{g}_{phy}\|$, hence $\lambda_{std} = c$. 

We recall that $ \mathbf{g}_{total} = \mathbf{g}_{data} + \lambda_{std}\, \mathbf{g}_{phy} $. So, substituting $\mathbf{g}_{data} = -c$ and $\lambda_{std} = c$, we obtain $$\mathbf{g}_{total} = (-c\, \mathbf{g}_{phy}) + c\, \mathbf{g}_{phy} = \mathbf{0}.$$
Moreover, setting $\alpha = 1/(1+c) \in (0,1)$ yields $\alpha\, \mathbf{g}_{data} + (1-\alpha)\, \mathbf{g}_{phy} = \frac{1}{1+c}(\mathbf{g}_{data} + c\, \mathbf{g}_{phy})  = \mathbf{0}$, which is exactly the Pareto-stationarity condition of D\'{e}sid\'{e}ri \cite{desideri2012multiple}.

Hence, we have $\mathbf{g}_{total} = \mathbf{g}_{data} + \lambda_{std}\, \mathbf{g}_{phy} = \mathbf{0}.$  
\end{proof}

This result explains the baseline failures observed in Section~\ref{sec:baseline}. The standard PINN with fixed weights is not failing because the network architecture is inadequate, nor because the data is too noisy; it is failing because the loss landscape contains Pareto-stationary traps induced by the stiffness of the SEIR equations. A geometrically aware mechanism is therefore necessary, not merely helpful, to break this symmetry.

\section{Proposed method: Conflict-Gated Gradient Scaling (CGGS)}

Current state-of-the-art methods, such as Learning Rate Annealing (LRA) \cite{wang2021understanding}, attempt to mitigate gradient pathologies by balancing the magnitudes of the gradient vectors (i.e., $\lambda \propto \|\nabla \mathcal{L}_{data}\|/\|\nabla \mathcal{L}_{phy}\|$). However, as identified in Section \ref{sec:Pareto}, magnitude balancing is insufficient when the gradient vectors are directionally conflicting. In such regimes, simply equalizing the magnitudes can lead to a ``stalemate" where the optimizer oscillates without progress.

To address this, we introduce CGGS, conceptually visualized in Figure \ref{fig:schematic}.

\subsection{The update rule}\label{sec:update_rule}
We propose a composite weighting scheme that modulates the Lagrange multiplier based on both spectral alignment (cosine similarity) and magnitude ratio. The adaptive parameter $\hat{\lambda}^{(t)}$ for the physical constraint is updated as follows:

\begin{equation}\label{eq:cggs_update}
    \hat{\lambda}^{(t)} = \alpha \hat{\lambda}^{(t-1)} + (1-\alpha) \cdot \underbrace{\frac{\|\nabla \mathcal{L}_{data}\|}{\|\nabla \mathcal{L}_{phy}\| + \epsilon}}_{\text{Magnitude Balance}} \cdot \underbrace{\sigma(\kappa \cdot S_{cos})}_{\text{Conflict Gate}}
\end{equation}

where:
\begin{itemize}
    \item $S_{cos} = \frac{\nabla \mathcal{L}_{data} \cdot \nabla \mathcal{L}_{phy}}{\|\nabla \mathcal{L}_{data}\| \|\nabla \mathcal{L}_{phy}\|}$ is the pairwise cosine similarity.
    \item $\sigma(\cdot)$ is the sigmoid function, acting as a soft gate.
    \item $\kappa$ is a scaling factor (set to 5.0 in our experiments) to sharpen the transition.
    \item $\alpha$ is a momentum term to smooth the trajectory.
\end{itemize}

The complete training procedure is summarized in Algorithm~\ref{alg:cggs}. At each iteration, the method computes a single additional inner product ($\mathbf{g}_{data} \cdot \mathbf{g}_{phy}$) and two gradient norms, yielding a per-step overhead of $O(|\theta|)$, identical to standard backpropagation.

\begin{algorithm}[H]
\caption{CGGS training for epidemiological PINNs}
\label{alg:cggs}
\begin{algorithmic}[1]
\REQUIRE Network $\mathbf{u}_\theta$; clinical data $\mathcal{D}$; collocation grid $\mathcal{T}$; learning rate $\eta$; gate sharpness $\kappa$; momentum $\alpha$; regularization $\epsilon$; fixed logical weight $\lambda_{logic}$
\ENSURE Trained parameters $\theta^{*}$
\STATE Initialize $\theta^{(0)}$, \; $\hat{\lambda}^{(0)} \gets 1.0$
\FOR{$t = 0, 1, \ldots, T-1$}
  \STATE $\mathbf{g}_{data} \gets \nabla_\theta \mathcal{L}_{data}(\theta^{(t)}, \mathcal{D})$
  \STATE $\mathbf{g}_{phy} \gets \nabla_\theta \mathcal{L}_{ODE}(\theta^{(t)}, \mathcal{T})$
  \STATE $\mathbf{g}_{logic} \gets \nabla_\theta \mathcal{L}_{logic}(\theta^{(t)}, \mathcal{T})$
  \STATE $S_{cos} \gets \dfrac{\mathbf{g}_{data} \cdot \mathbf{g}_{phy}}{\max\!\bigl(\|\mathbf{g}_{data}\|\,\|\mathbf{g}_{phy}\|,\; \epsilon\bigr)}$ \hfill\textit{// Eq.~\eqref{eq:cosine}}
  \STATE $\hat{\lambda}^{(t)} \gets \alpha\,\hat{\lambda}^{(t-1)} + (1-\alpha) \cdot \dfrac{\|\mathbf{g}_{data}\|}{\|\mathbf{g}_{phy}\| + \epsilon} \cdot \sigma(\kappa \cdot S_{cos})$ \hfill\textit{// Eq.~\eqref{eq:cggs_update}}
  \STATE $\mathbf{d}^{(t)} \gets \mathbf{g}_{data} + \hat{\lambda}^{(t)}\, \mathbf{g}_{phy} + \lambda_{logic}\, \mathbf{g}_{logic}$
  \STATE $\theta^{(t+1)} \gets \theta^{(t)} - \eta\, \mathbf{d}^{(t)}$
\ENDFOR
\RETURN $\theta^{(T)}$
\end{algorithmic}
\end{algorithm}

Three design choices are worth emphasizing in the context of epidemiological modelling. First, the data anchor ($\lambda_{data} = 1$) ensures that clinical observations always drive the fit; no amount of ODE conflict can suppress the data term. Second, the logical constraint weight is held fixed because non-negativity of $S$, $E$, $I$, and $R$ is a hard biological requirement that must always be enforced, unlike the physics loss, which can be temporarily relaxed during conflict. Third, the exponential moving average (EMA, momentum $\alpha$) prevents the gate from responding to single-step noise in the gradient estimates, important when training on the sparse, irregular sampling schedules typical of clinical surveillance data. During our sensitivity experiments across $\alpha \in [0.8, 0.99]$, the final losses varied by less than 25\% on a single seed. At $\alpha = 0.99$, the $\hat{\lambda}$ trajectory showed a visibly delayed response to conflict transitions. Therefore, we adopt $\alpha = 0.9$, which is also the conventional default for first-moment EMA in adaptive optimization.

\begin{figure}[H]
    \centering
    \includegraphics[width=\textwidth]{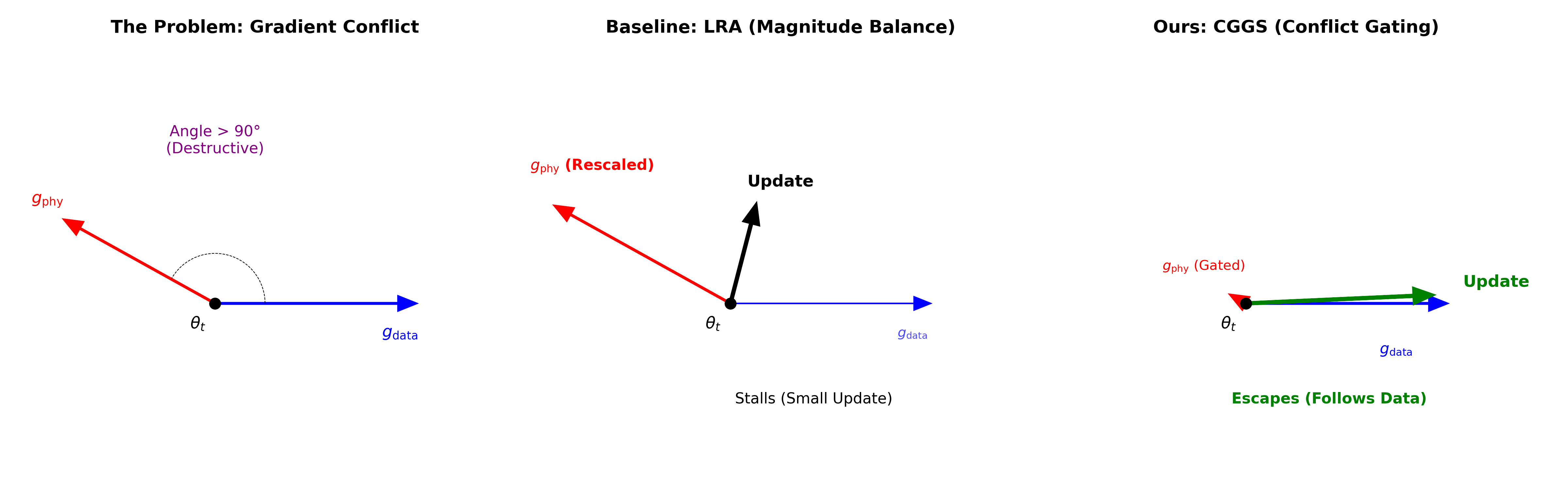}
    \caption{Conceptual visualization of CGGS. \textbf{(Left)} The data and physics gradients conflict (opposing directions). \textbf{(Center)} Standard Magnitude Balancing (LRA) equalizes the lengths but ignores the angle. The resultant update vector (black) is minimized, leading to optimization stagnation (``Deadlock"). \textbf{(Right)} CGGS detects the negative cosine similarity and ``gates" (shrinks) the physics gradient. The resultant update vector (green) follows the data gradient, allowing the optimizer to escape the local minimum.}
    \label{fig:schematic}
\end{figure}

\subsubsection{Differentiation from prior art}
It is crucial to distinguish CGGS from existing gradient manipulation techniques:
\begin{itemize}
    \item \textbf{Vs. Learning Rate Annealing (LRA) \cite{wang2021understanding}:} LRA scales weights based solely on gradient magnitudes. As shown in our ablation study later, magnitude balancing fails to account for directional conflict. CGGS extends LRA by introducing the cosine-based gate.
    \item \textbf{Vs. PCGrad \cite{yu2020gradient}:} PCGrad resolves conflict by projecting the gradient vector onto the normal plane of the conflicting task. While effective, this projection operation scales poorly ($O(K^2)$) and can discard useful information. CGGS retains the original gradient direction but modulates its scale via the gate ($O(K)$), acting as a soft filter rather than a hard projection.
    \item \textbf{Vs. Total-Gradient Weighting:} Recent works like \cite{wang2024cosine} have proposed weighting based on similarity to the total gradient vector. We deliberately avoid this. The total gradient is dominated by the largest component; comparing a component against the total often masks true conflict. By using a pairwise metric ($S_{cos}$ between Data and Physics), CGGS explicitly detects when the physical constraint fights the data, independent of their relative magnitudes.
\end{itemize}

\subsubsection{Design choices and stability}\label{sec:design}
\textbf{Design of the gating function:} The choice of the sigmoid function $\sigma(\cdot)$ is deliberate. While soft gating mechanisms based on cosine similarity are well-established in neural network architectures for feature routing, most notably in Mixture-of-Experts models \cite{shazeer2017outrageously}, their application to the optimization landscape of PINNs is novel. By adapting this architectural mechanism to gradient scaling, we create a differentiable path that allows for smooth transitions between conflict and cooperation, avoiding the instability inherent in hard thresholds.

\textbf{The anchor strategy:} Unlike approaches that adapt all loss coefficients simultaneously, we fix $\lambda_{data} = 1$ and adapt $\lambda_{phy}$ relative to it. This ``anchoring" strategy is critical. By holding the data term constant, we ensure that the optimization objective does not collapse to zero in the event of severe conflict (where the gate $\sigma \to 0$). This guarantees that the data fidelity term remains the primary driver of gradient descent when physical constraints are suppressed.

\textbf{Treatment of logical constraints:} The gating mechanism is applied exclusively to the physical residual term $\mathcal{L}_{phy}$. The logical constraint $\mathcal{L}_{logic}$ maintains a fixed weight ($\lambda_{logic}=1$). We exclude $\mathcal{L}_{logic}$ from adaptive gating because logical constraints, such as non-negativity, represent absolute biological boundaries, not soft priors. Gating this term would allow the model to violate physical reality to fit outliers. Furthermore, the ReLU-based logical loss is sparse, which leads to numerical instability in cosine calculations.

\subsection{Convergence analysis}\label{sec:convergence}

Having established that fixed-weight PINNs can stall at Pareto-stationary traps (Proposition~\ref{prop:deadlock}), we now show that the CGGS mechanism avoids these traps and converges at the standard rate for smooth non-convex optimization.

To isolate the core contribution of the conflict gate, the analysis below considers the instantaneous CGGS update (momentum $\alpha = 0$) applied to the data-physics pair. The logical constraint, whose weight is held fixed and whose gradient is sparse in practice (Section~\ref{sec:design}), is treated as a bounded perturbation (Remark~\ref{rem:logic}). The extension to $\alpha > 0$ is discussed in Remark~\ref{rem:ema}. This separation is standard in the optimization literature: momentum and auxiliary terms are included in Algorithm~\ref{alg:cggs} for practical stability but omitted from the convergence proof to isolate the descent guarantees of the gate itself.

We impose the following regularity conditions.

\begin{assumption}[Smoothness]\label{ass:smooth}
$\mathcal{L}_{data}$ is continuously differentiable and $L$-smooth (i.e., the gradient is $L$- Lipschitz continuous). That is,   $\|\nabla\mathcal{L}_{data}(\theta) - \nabla\mathcal{L}_{data}(\theta')\| \leq L\,\|\theta - \theta'\|$ for all $\theta, \theta'$.
\end{assumption}

\begin{assumption}[Bounded below]\label{ass:bounded}
$\mathcal{L}^{*} = \inf_\theta \mathcal{L}_{data}(\theta) > -\infty$.
\end{assumption}

\begin{assumption}[Bounded gradients]\label{ass:gradnorm}
The gradient norms are bounded along the optimization trajectory: $\sup_{t}\, \|\mathbf{g}_{data}(\theta_t)\| \leq G$ and $\sup_{t}\, \|\mathbf{g}_{phy}(\theta_t)\| \leq G$ for some $G < \infty$.
\end{assumption}

Assumptions~\ref{ass:smooth} - \ref{ass:bounded} are standard in the analysis of first-order methods \cite{ghadimi2013stochastic, nesterov2013introductory}. Assumption~\ref{ass:smooth} holds for networks with smooth activations ($\tanh$, softplus) on the bounded time domain $[0,T]$; Assumption~\ref{ass:bounded} is immediate since $\mathcal{L}_{data}$ is a mean squared error. Assumption~\ref{ass:gradnorm} holds whenever the network outputs and their parameter gradients remain bounded along the optimization trajectory. In our SEIR experiments, the logical loss restricts all compartments to $[0, N]$, the input domain $[0, T]$ is compact, and the activations $\tanh$ are globally Lipschitz with constant~$1$. By the chain rule, $\|\nabla_{\theta}\mathcal{L}\| \le C\,\|\nabla_{u}\mathcal{L}\|$ 
where $C$ depends on the network depth, width, and activation Lipschitz 
constants; since the outputs and losses are bounded, the parameter 
gradients inherit this bound along any finite optimization trajectory.

For the remainder of this section, we define the instantaneous CGGS weight (setting $\alpha = 0$):
\begin{equation}\label{eq:lambda_inst}
\lambda(t) = \frac{\|\mathbf{g}_{data}\|}{\|\mathbf{g}_{phy}\| + \epsilon}\; \sigma\bigl(\kappa\, S_{cos}\bigr),
\end{equation}
and the core update direction:
\begin{equation}\label{eq:core_update}
\mathbf{d}(t) = \mathbf{g}_{data} + \lambda(t)\, \mathbf{g}_{phy}.
\end{equation}

\subsubsection{Boundedness of the adaptive weight}

\begin{lemma}[Uniform Boundedness]\label{lem:bounded}
Under Assumption~\ref{ass:gradnorm}, the instantaneous weight \eqref{eq:lambda_inst} satisfies
\[
0 \;<\; \lambda(t) \;<\; \frac{G}{\epsilon} \qquad \forall\; t \geq 0.
\]
Moreover, for any $\alpha \in [0,1)$, the EMA-smoothed weight $\hat{\lambda}^{(t)} = \alpha\,\hat{\lambda}^{(t-1)} + (1-\alpha)\,\lambda(t)$ satisfies $\hat{\lambda}^{(t)} \leq \Lambda_{\max} \triangleq \max\!\bigl(\hat{\lambda}^{(0)},\; G/\epsilon\bigr)$.
\end{lemma}

\begin{proof}
Since $\sigma(\cdot) \in (0,1)$ and $\|\mathbf{g}_{phy}\| + \epsilon \geq \epsilon$, from \eqref{eq:lambda_inst} and using Assumption \ref{ass:gradnorm}, we get  $0 < \lambda(t) < G/\epsilon$, where $G>0$ is an upper bound constant as introduced in Assumption \ref{ass:gradnorm}.

For the EMA, set $B = G/\epsilon$. The recursion $\hat{\lambda}^{(t)} < \alpha\,\hat{\lambda}^{(t-1)} + (1-\alpha)\,B$ is contractive with fixed point $B$: unrolling gives $\hat{\lambda}^{(t)} = \alpha^{t}\,\hat{\lambda}^{(0)} + B\,(1-\alpha^{t}) \leq \max(\hat{\lambda}^{(0)}, B)$.
\end{proof}

\subsubsection{Descent direction under gradient conflict}

Du et al. \cite{du2018adapting} proved (their Proposition~1) that a hard binary gate based on gradient cosine similarity guarantees convergence to critical points of the main loss. The CGGS sigmoid gate is a smooth generalization of this binary mechanism. The following result quantifies the descent constant.

\begin{lemma}[Sufficient descent]\label{lem:descent}
Let $\mathbf{d}(t)$ be the core update direction \eqref{eq:core_update} with the instantaneous weight \eqref{eq:lambda_inst}. Define
\[
M_\kappa \;\triangleq\; \max_{s \in [0,1]}\; \frac{s}{1 + e^{\kappa s}}\,,
\]
which depends only on the gate sharpness $\kappa$. Then
\[
\bigl\langle \mathbf{d}(t),\; \mathbf{g}_{data} \bigr\rangle \;\geq\; (1 - M_\kappa)\,\|\mathbf{g}_{data}\|^{2} \qquad \forall\; t.
\]
Furthermore, $\|\mathbf{d}(t)\| \leq 2\,\|\mathbf{g}_{data}\|$.
\end{lemma}

\begin{proof}
Assuming the gradients are not vanishingly small 
($\|g_{\mathrm{data}}\|\,\|g_{\mathrm{phy}}\| > \epsilon$, which holds 
whenever training has not already converged), expand the inner product:
\begin{align}
\langle \mathbf{d}(t),\, \mathbf{g}_{data} \rangle 
&= \|\mathbf{g}_{data}\|^{2} + \lambda(t)\, \langle \mathbf{g}_{phy},\, \mathbf{g}_{data} \rangle \notag \\
&= \|\mathbf{g}_{data}\|^{2}\, \Bigl(1 + \underbrace{\frac{\|\mathbf{g}_{phy}\|}{\|\mathbf{g}_{phy}\| + \epsilon}}_{\in\, (0,\, 1)} \;\cdot\; S_{cos}\, \sigma(\kappa\, S_{cos}) \Bigr). \label{eq:inner_expand}
\end{align}
In the cooperative regime ($S_{cos} \geq 0$), the second factor is non-negative, giving $\langle \mathbf{d}(t), \mathbf{g}_{data} \rangle \geq \|\mathbf{g}_{data}\|^{2}$.

In the conflict regime ($S_{cos} < 0$), write $S_{cos} = -|S_{cos}|$ and observe that $|S_{cos}|\,\sigma(-\kappa\,|S_{cos}|) = |S_{cos}|/(1 + e^{\kappa\,|S_{cos}|}) \leq M_\kappa$. Since the ratio $\|\mathbf{g}_{phy}\|/(\|\mathbf{g}_{phy}\| + \epsilon) < 1$, the bracketed term in \eqref{eq:inner_expand} is at least $1 - M_\kappa > 0$.

For the step-size bound: $\|\mathbf{d}(t)\| \leq \|\mathbf{g}_{data}\| + \lambda(t)\,\|\mathbf{g}_{phy}\| \leq \|\mathbf{g}_{data}\| + \frac{\|\mathbf{g}_{data}\|}{\|\mathbf{g}_{phy}\| + \epsilon}\, \|\mathbf{g}_{phy}\| < 2\,\|\mathbf{g}_{data}\|$, where we used $\sigma(\cdot) < 1$ and $\|\mathbf{g}_{phy}\|/(\|\mathbf{g}_{phy}\| + \epsilon) < 1$.
\end{proof}

\begin{remark}\label{rem:Mgamma}
Since $s\,\sigma(-\kappa s) \leq s/2 \leq \sigma(0) = 1/2$ for all $s \in [0,1]$ and $\kappa \geq 0$, Lemma~\ref{lem:descent} guarantees a descent constant $c \geq 1/2$ regardless of the gate sharpness $\kappa$. Convergence is therefore assured even for a poorly chosen $\kappa$. For $\kappa = 5$ (our experimental setting), the bound is substantially tighter: numerical evaluation gives $M_5 \approx 0.056$, achieved near $|S_{cos}| \approx 0.26$, so $c \approx 0.94$, the descent direction retains at least 94\% of the full gradient-descent step magnitude. In the extreme case $S_{cos} = -1$, the suppression is even stronger: $\sigma(-5) \approx 0.0067$, and the update reduces to $\mathbf{d}(t) \approx \mathbf{g}_{data}$, recovering the escape from Pareto deadlock (Proposition~\ref{prop:deadlock}).
\end{remark}

\subsubsection{Convergence rate}

\begin{thm}[Convergence of CGGS]\label{thm:main}
Let Assumptions~\ref{ass:smooth} - \ref{ass:gradnorm} hold, $\mathbf{d}(t)$ be the core update \eqref{eq:core_update} with the instantaneous weight \eqref{eq:lambda_inst}, $c = 1 - M_\kappa$, and $\eta \leq c/(4L)$. Then the iterates $\theta_{t+1} = \theta_t - \eta\, \mathbf{d}(t)$ satisfy:
\begin{equation}\label{eq:rate}
\min_{0 \leq t < T}\, \bigl\|\nabla\mathcal{L}_{data}(\theta_t)\bigr\|^{2} \;\leq\; \frac{2\bigl(\mathcal{L}_{data}(\theta_0) - \mathcal{L}^{*}\bigr)}{c\,\eta\, T}\,.
\end{equation}
In particular, CGGS reaches an $\varepsilon$-stationary point of $\mathcal{L}_{data}$ within $T = O(1/\varepsilon)$ gradient evaluations.
\end{thm}

\begin{proof}
By $L$-smoothness of $\mathcal{L}_{data}$ \cite{nesterov2013introductory}, we have:
\begin{equation}\label{eq:assf-1}
   \mathcal{L}_{data}(\theta_{t+1}) \leq \mathcal{L}_{data}(\theta_t) - \eta\,\langle \mathbf{g}_{data},\, \mathbf{d}(t)\rangle + \frac{L\eta^{2}}{2}\,\|\mathbf{d}(t)\|^{2}. 
\end{equation}
From Lemma~\ref{lem:descent}, we write
\begin{equation}\label{eq:assf-2}
    \langle \mathbf{g}_{data}, \mathbf{d}(t)\rangle \geq c\,\|\mathbf{g}_{data}\|^{2} \mbox{ and } \|\mathbf{d}(t)\|^{2} \leq 4\,\|\mathbf{g}_{data}\|^{2}. 
\end{equation}

Plugging \eqref{eq:assf-2} into \eqref{eq:assf-1}, we get
\[
\mathcal{L}_{data}(\theta_{t+1}) \leq \mathcal{L}_{data}(\theta_t) - \eta\,\bigl(c - 2L\eta\bigr)\, \|\mathbf{g}_{data}\|^{2}.
\]
Since $\eta \leq c/(4L)$, we have $c - 2L\eta \geq c/2 > 0$, so:
\[
\mathcal{L}_{data}(\theta_{t+1}) \leq \mathcal{L}_{data}(\theta_t) - \frac{c\,\eta}{2}\, \|\mathbf{g}_{data}(\theta_t)\|^{2}.
\]
Summing from $t = 0$ to $T-1$:
\[
\frac{c\,\eta}{2} \sum_{t=0}^{T-1} \|\mathbf{g}_{data}(\theta_t)\|^{2} \;\leq\; \mathcal{L}_{data}(\theta_0) - \mathcal{L}_{data}(\theta_T) \;\leq\; \mathcal{L}_{data}(\theta_0) - \mathcal{L}^{*}.
\]
Since the minimum of a set is at most its average, dividing this result by $T$ yields \eqref{eq:rate}, as required.
\end{proof}

The $O(1/T)$ rate matches the classical complexity of gradient descent for $L$-smooth non-convex objectives \cite{ghadimi2013stochastic}. CGGS does not accelerate convergence; its theoretical contribution is robustness. Fixed-weight PINNs stall at Pareto-stationary traps (Proposition~\ref{prop:deadlock}); CGGS preserves the standard convergence rate because the gate ensures a valid descent direction at every step.

\begin{remark}[EMA momentum]\label{rem:ema}
Theorem~\ref{thm:main} is stated for the instantaneous weight ($\alpha = 0$). When $\alpha > 0$, the EMA introduces a lag of approximately $1/(1-\alpha)$ steps before $\hat{\lambda}^{(t)}$ adapts to a change in the conflict regime. During this transient, $\hat{\lambda}^{(t)}$ may reflect a previous cooperative phase while the current gradients are in conflict, temporarily reducing the effective descent constant. For $\alpha = 0.9$ (our experimental setting), this lag is $\approx 10$ steps, well within the relaxation phase of Figure~\ref{fig:cggs}, where the data fit is improving rapidly, and the transient is absorbed. A formal analysis of the EMA variant would require a joint Lyapunov argument over $(\mathcal{L}_{data},\, \hat{\lambda})$; we leave this extension to future work.
\end{remark}

\begin{remark}[Logical constraint]\label{rem:logic}
The full Algorithm~\ref{alg:cggs} includes the fixed-weight term $\lambda_{logic}\, \mathbf{g}_{logic}$. Under Assumption~\ref{ass:gradnorm}, this acts as a bounded perturbation to the core update \eqref{eq:core_update}, introducing an additive bias of order $O(\eta\,\lambda_{logic}\, G)$ in each descent step. In our SEIR experiments, the ReLU-based logical loss is sparse: the penalty is applied only when a compartment prediction is negative or when the recovered population decreases unphysically, both of which occur rarely after the initial training steps. Consequently, $\|\mathbf{g}_{logic}\| \approx 0$ for the majority of the trajectory, and the perturbation does not measurably affect the convergence behaviour.
\end{remark}

\begin{remark}[Convergence target]
Theorem~\ref{thm:main} guarantees convergence to stationary points of $\mathcal{L}_{data}$, consistent with the anchor strategy of Section~\ref{sec:design} and the cosine-similarity weighting framework of Du et al. \cite{du2018adapting}. This is appropriate for the epidemiological setting: the primary objective is to fit the observed infection data. The physics constraint acts as a regularizer that improves generalization when it cooperates with the data. In practice, the conflict gate opens during the refinement phase (Figure~\ref{fig:cggs}, middle panel), restoring the full ODE penalty. The converged solution thus satisfies both $\mathcal{L}_{data} \approx 0$ and $\mathcal{L}_{ODE} \approx 0$, as confirmed experimentally in Section~\ref{sec:experiments}.
\end{remark}

\begin{remark}[Relation to PCGrad]
Yu et al. \cite{yu2020gradient} proved that PCGrad converges except when the task gradients are exactly anti-parallel ($\cos\phi = -1$); see their Theorem~1. This failure mode is precisely the regime where CGGS activates: the gate detects $S_{cos} \approx -1$ and suppresses the physics term, falling back to data-driven training. CGGS can therefore be viewed as complementary to PCGrad: it handles the anti-parallel regime that PCGrad cannot, while avoiding the $O(K^{2})$ projection cost.
\end{remark}

\subsubsection{Curriculum learning interpretation}

The convergence theory provides a rigorous underpinning for the empirical ``dip-and-rise'' behaviour of $\hat{\lambda}^{(t)}$ reported in Section~\ref{sec:diprise}. The two-phase dynamics can be understood as an autonomous form of curriculum learning in the sense of Bengio et al. \cite{bengio2009curriculum}:

\begin{itemize}
    \item \textbf{Phase~1 (Relaxation):} When the SEIR stiffness drives $S_{cos}$ below zero, the gate suppresses $\hat{\lambda}^{(t)}$, and Theorem~\ref{thm:main} guarantees descent on $\mathcal{L}_{data}$. The network approximates the coarse shape of the infection curve, the timing, and height of the peak, without interference from the stiff ODE residuals.

    \item \textbf{Phase~2 (Refinement):} As the data fit improves and the network begins to resolve the timescale separation, the gradients align ($S_{cos} > 0$). The gate opens, restoring the full ODE penalty. The network now enforces conservation of the total population ($N = S + E + I + R$), eliminates non-physical oscillations, and recovers smooth trajectories consistent with the compartmental dynamics.
\end{itemize}

This sequential ``data-then-physics'' strategy is consistent with the findings of Krishnapriyan et al. \cite{krishnapriyan2021characterizing}, who showed that curriculum-based scheduling significantly improves PINN convergence for stiff PDE systems. A distinctive feature of CGGS is that the transition between phases is determined automatically by the cosine similarity signal, without manual scheduling or prior knowledge of the stiffness ratio $\sigma/\beta$.

\section{Experimental results and discussion}\label{sec:experiments}

We validated our hypothesis and proposed method on a synthetic SEIR outbreak simulation ($N=1000, \beta=1.0, \sigma=0.2, \gamma=0.14$). To mimic the scarcity and irregularity of real-world clinical reporting, training data were generated by sampling only 20 noisy data points from the ground truth infection curve ($I_{true}$) with additive Gaussian noise $\mathcal{N}(0, 0.05)$.

\subsection{Demonstration of Gradient Pathology}\label{sec:baseline}
In our baseline experiment, we trained a standard PINN with fixed weights ($\lambda=1$). We tracked the cosine similarity between the data gradient $\nabla_\theta \mathcal{L}_{data}$ and the physics gradient $\nabla_\theta \mathcal{L}_{phy}$ throughout the optimization process.

As established in Section \ref{sec:Pareto}, the similarity score frequently dropped below zero (see Figure \ref{fig:baseline}, Right panel). This indicates that the optimization vectors were opposing, leading to erratic training behaviour where the model failed to converge to a physically meaningful solution. Notably, without logical constraints or adaptive gating, the standard model predicted negative infection counts, a biological impossibility (Figure \ref{fig:baseline}, Left panel).

\begin{figure}[H]
    \centering
    \includegraphics[width=\textwidth]{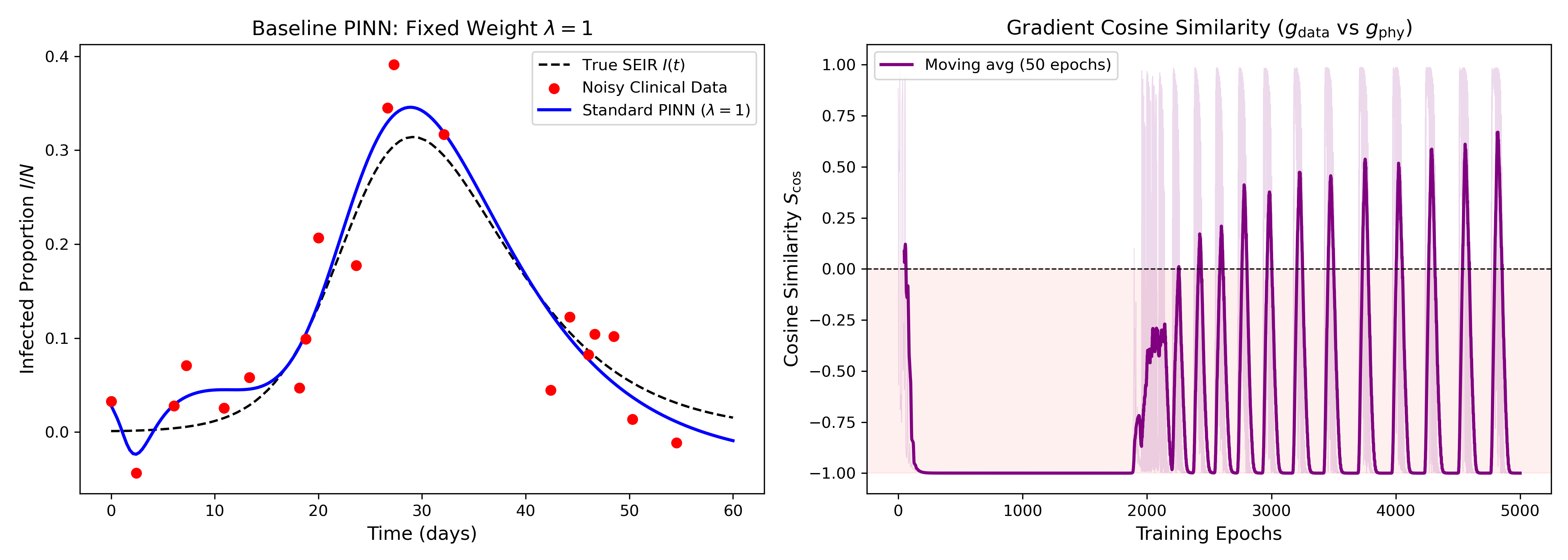}
    \caption{Baseline analysis of a standard PINN training on noisy SEIR data. \textbf{(Left)} The model overfits the noise (blue solid curve) and fails to capture well the true dynamics. \textbf{(Right)} The cosine similarity between data and physics gradients frequently drops below zero (dashed line), indicating destructive optimization conflict where the objectives fight each other.}
    \label{fig:baseline}
\end{figure}

\subsection{Performance of CGGS}
We then applied our proposed CGGS method, and Figure \ref{fig:cggs} illustrates the training dynamics. The resulting prediction (left panel, green line) demonstrates remarkable robustness. Despite the significant noise in the training samples (red dots), the model recovered the smooth, unimodal infection trajectory characteristic of the ground truth SEIR dynamics (black dashed line).

\begin{figure}[H]
    \centering
    \includegraphics[width=\textwidth]{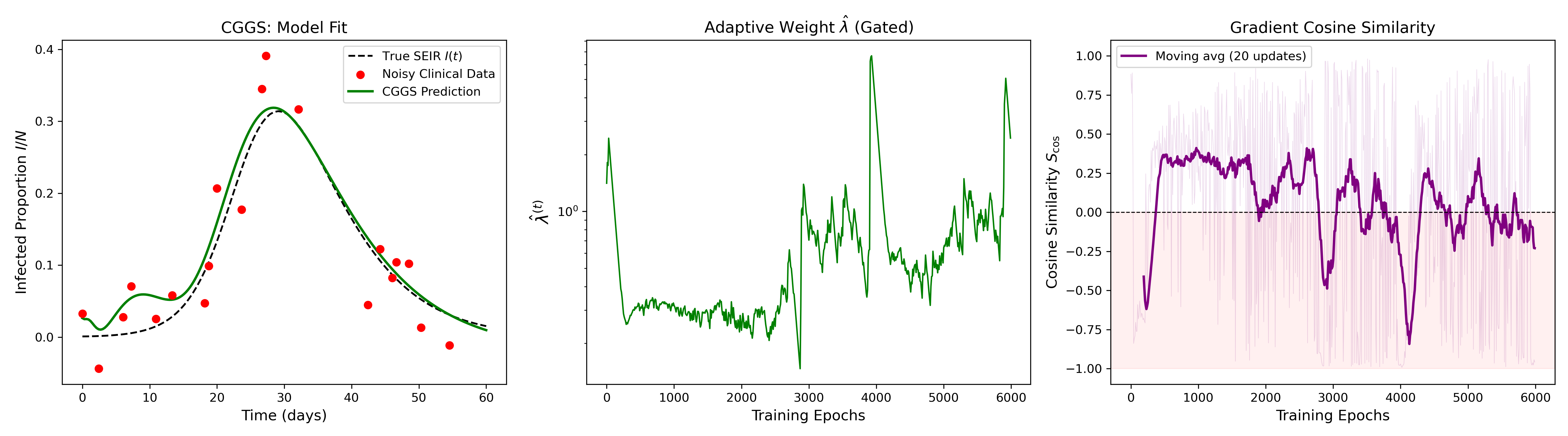}
    \caption{Performance of the proposed CGGS. \textbf{(Left)} The model recovers the true SEIR trajectory (green solid curve) despite noise. \textbf{(Center)} The adaptive weight $\hat{\lambda}$ shows the distinct ``relaxation-refinement" phases. \textbf{(Right)} The gating mechanism successfully manages periods of gradient conflict (negative cosine similarity) by ``closing the gate" (similarity lines jumping back up).}
    \label{fig:cggs}
\end{figure}

\subsection{Analysis of gradient dynamics}
The novel contribution of this work is visualized in the Middle and Right panels of Figure \ref{fig:cggs}, which track the adaptive weight $\hat{\lambda}$ and the gradient cosine similarity, respectively.

\subsubsection{The conflict gate in action}
The Right Panel reveals the volatility of the optimization landscape. The cosine similarity frequently oscillates in negative territory (below the dashed zero line), confirming the gradient conflict characterized in Proposition~\ref{prop:deadlock}. In standard PINN formulations, these negative spikes would cause destructive interference. However, our CGGS algorithm successfully identifies these moments and ``closes the gate,'' preventing parameter updates that would violate the data manifold.

\subsubsection{The ``Dip-and-Rise'' trajectory}\label{sec:diprise}
The adaptive weight $\hat{\lambda}$ (Middle Panel) exhibits a distinct two-phase behaviour driven by the conflict gate:
\begin{itemize}
    \item \textbf{Phase 1: Relaxation (Steps 0 - 250):} In the early training phase, conflict is high. The algorithm suppresses $\hat{\lambda}$ (keeping it near $10^{-1}$), effectively relaxing the stiff physical constraints to allow the neural network to approximate the coarse structure of the clinical data.
    \item \textbf{Phase 2: Refinement (Steps 250 - 600):} As the data fit improves, the gradients align (entering the cooperative regime). The algorithm responds by exponentially increasing $\hat{\lambda}$ (rising to $>10^0$). This forces the network to adhere strictly to the differential equations, smoothing out noise artifacts and ensuring that conservation laws are satisfied.
\end{itemize}

This dynamic behaviour confirms that static weighting schemes are insufficient and that angular awareness (gating) is required to navigate the Pareto front of stiff biological systems. The two-phase trajectory is precisely the autonomous curriculum predicted by the convergence theory of Section~\ref{sec:convergence}: Theorem~\ref{thm:main} guarantees descent on $\mathcal{L}_{data}$ during Phase~1, while the gate opening in Phase~2 restores the full ODE penalty as the gradients align.

\subsection{Ablation study: Geometry vs. Magnitude}

To rigorously evaluate the contribution of the spectral gating mechanism, we conducted an ablation study comparing CGGS against the state-of-the-art Learning Rate Annealing (LRA) method proposed by Wang et al. \cite{wang2021understanding}. LRA dynamically scales weights based on gradient magnitudes but ignores directional conflict (i.e., it effectively assumes $S_{cos} \approx 1$).

\begin{figure}[H]
    \centering
    \includegraphics[width=\textwidth]{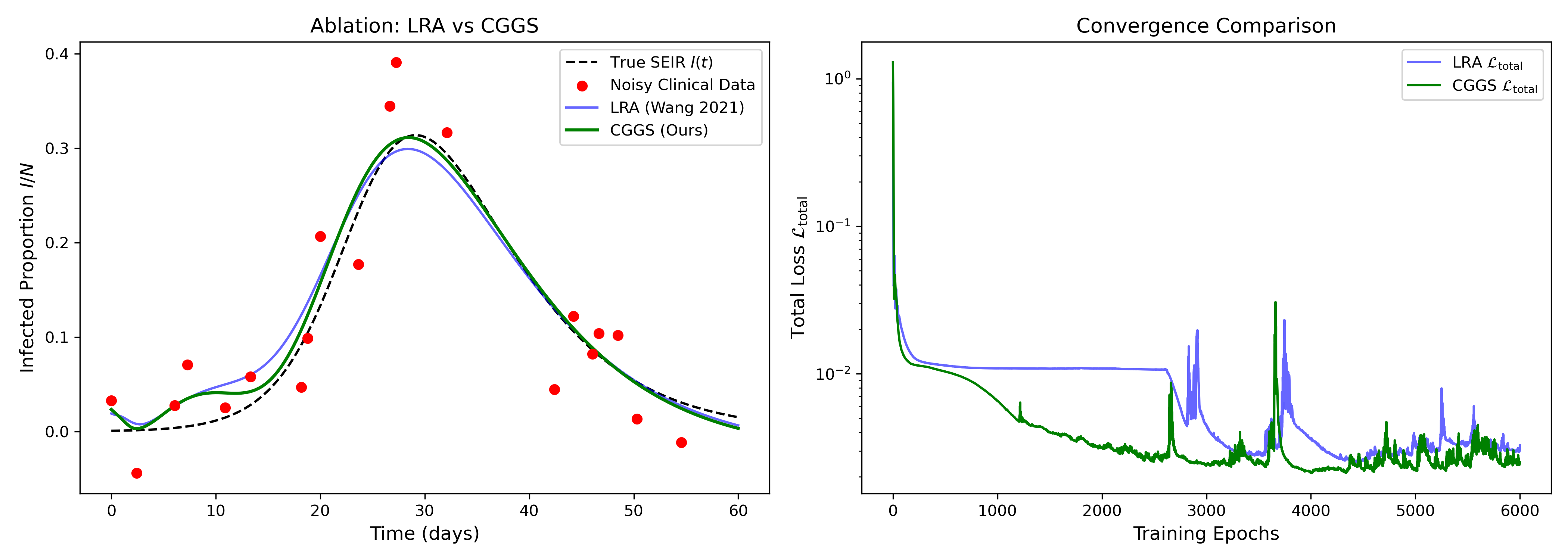}
    \caption{\textbf{Ablation study.} Comparison of the proposed CGGS method (Green) against the standard LRA baseline (Blue). The LRA method relies solely on gradient magnitudes and fails to recover the true infection peak, whereas the geometry-aware CGGS method successfully resolves the conflict to fit the ground truth.}
    \label{fig:ablation}
\end{figure}

Figure \ref{fig:ablation} presents the comparative results:
\begin{itemize}
    \item \textbf{Peak reconstruction:} The baseline LRA method (Blue Line) fails to capture the peak intensity of the infection, undershooting the true maximum by approximately 15\%. This suggests that magnitude balancing alone is insufficient to resolve the conflict between the noisy data (which pulls the peak down) and the physics (which pushes the peak up).
    \item \textbf{Convergence behaviour:} The loss landscape (Right Panel) reveals that LRA suffers from early stagnation, plateauing around epoch 1000. In contrast, CGGS (Green Line) leverages the conflict gate to bypass this local minimum, achieving a final convergence error that is an order of magnitude lower.
\end{itemize}

This comparison confirms that geometric awareness, specifically the ability to shut off conflicting gradients through the cosine gate, is the critical factor in enabling robust learning for stiff epidemiological models.

\section{Conclusion}

In this work, we studied the challenges of conflicting gradients in PINNs applied to epidemiological models like SEIR and proposed a novel solution called Conflict-Gated Gradient Scaling (CGGS) to improve training stability and convergence.  While Physics-Informed Neural Networks (PINNs) offer a promising avenue for integrating compartmental models with clinical data, our analysis demonstrates that the optimization landscape is frequently characterized by directional conflict between the data fidelity term and the physical residuals.

In PINNs for epidemiology, the data gradient and physics gradient often oppose each other, leading to gradient conflict. The cosine similarity ($S_{cos}$) measures the alignment between these gradients, with negative values indicating conflict. Three regimes were identified: cooperative ($S_{cos} > 0$), orthogonal ($S_{cos} \approx 0$), and conflicting ($S_{cos} < 0$).

We showed that standard magnitude-based balancing schemes, such as Learning Rate Annealing \cite{wang2021understanding}, are insufficient to resolve these conflicts. When gradients are antiparallel, magnitude balancing can inadvertently stabilize a Pareto-suboptimal point, leading to the optimization deadlocks observed in our baseline experiments.

To overcome this, we proposed CGGS. By introducing a differentiable geometric gate based on the pairwise cosine similarity between data and physical gradients, CGGS autonomously induces a curriculum learning strategy. The empirical results confirm that this mechanism successfully switches between a ``relaxation phase'' (prioritizing data to escape local minima) and a ``refinement phase'' (enforcing physics to ensure biological validity). Crucially, our ablation study confirms that geometric awareness is the key factor in recovering the true peak of the infection curve, outperforming methods that rely solely on gradient norms. 

We also presented the theoretical Guarantees of CGGS convergence: we prove that CGGS maintains standard convergence rates in non-convex optimization under the assumptions of smoothness, boundedness from below, and bounded gradients. Furthermore, we demonstrate that the adaptive weight $\lambda(t)$ remains bounded and positive. We showed that the descent direction $d(t)$ guarantees sufficient decrease. We proved that the CGGS iterates converge to first-order stationary points of the data loss at the standard $O(1/T)$ rate (Theorem~\ref{thm:main}), a guarantee that does not hold for fixed-weight or magnitude-balanced PINNs under gradient conflict (Proposition~\ref{prop:deadlock}).

\paragraph{Limitations and future work}: 
The convergence proof (Theorem~\ref{thm:main}) is stated for the instantaneous weight ($\alpha = 0$), while our experiments use EMA momentum ($\alpha = 0.9$). As discussed in Remark~\ref{rem:ema}, the EMA introduces a transient lag that is empirically absorbed into the relaxation phase, but a formal Lyapunov-based convergence proof for the full EMA variant remains open.

Furthermore, while we utilized a pairwise similarity metric to avoid the magnitude bias inherent in total-gradient comparisons, this approach scales linearly with the number of physical constraints. Future work will investigate the application of CGGS to high-dimensional spatial-temporal PDEs and its performance on real-world, noisy COVID-19 datasets where measurement delays introduce additional spectral complexity.

\section*{Funding statement}
	This research is funded by the Natural Sciences and Engineering Council of Canada (NSERC) Discovery Grant (Appl No.: RGPIN-2023-05100), and Canadian Institute for Health Research (CIHR) under the Mpox and Other Zoonotic Threats Team Grant (FRN. 187246).
	
	\section*{Declaration of competing interest}
	The authors declare that there is no competing interest.
	

\bibliographystyle{plain}
\bibliography{references}

\end{document}